\pdfoutput=1

\documentclass[11pt]{article}

\usepackage[]{acl}

\usepackage{times}
\usepackage{latexsym}
\usepackage{amsmath}
\usepackage{amsthm}
\usepackage{booktabs}
\usepackage{multirow}
\usepackage{graphicx}
\usepackage{threeparttable}
\newtheorem{theorem}{Theorem}
\usepackage{enumitem}
\usepackage{float}

\usepackage[T1]{fontenc}
\usepackage{array}
\newcolumntype{C}[1]{>{\centering\arraybackslash}p{#1}}
\newcolumntype{P}{>{\centering\arraybackslash}p{2.5em}}

\usepackage[utf8]{inputenc}

\usepackage{microtype}

\usepackage[normalem]{ulem}

\title{Domain Representative Keywords Selection: A Probabilistic Approach}

\author{Pritom Saha Akash$^{1}$ $\quad$ Jie Huang$^{1}$ $\quad$ Kevin Chen-Chuan Chang$^{1}$ $\quad$ \\ {\bf Yunyao Li$^{2}$\thanks{ \textsuperscript{\textasteriskcentered} This work was done while the author was at IBM Research, USA.} } $\quad$ {\bf Lucian Popa$^{3}$} $\quad$ {\bf ChengXiang Zhai$^{1}$} \\
 $^1$University of Illinois at Urbana-Champaign, USA \\
 $^2$Apple, USA\\
 $^3$IBM Research, USA \\
 \texttt{\{pakash2, jeffhj, kcchang, czhai\}@illinois.edu} \\
 \texttt{yunyaoli@apple.com, lpopa@us.ibm.com}
}

\begin{document}
\maketitle
\begin{abstract}
We propose a probabilistic approach to select a subset of a \textit{target domain representative keywords} from a candidate set, contrasting with a context domain. Such a task is crucial for many downstream tasks in natural language processing. To contrast the target domain and the context domain, we adapt the \textit{two-component mixture model} concept to generate a distribution of candidate keywords. It provides more importance to the \textit{distinctive} keywords of the target domain than common keywords contrasting with the context domain. To support the \textit{representativeness} of the selected keywords towards the target domain, we introduce an \textit{optimization algorithm} for selecting the subset from the generated candidate distribution. We have shown that the optimization algorithm can be efficiently implemented with a near-optimal approximation guarantee. Finally, extensive experiments on multiple domains demonstrate the superiority of our approach over other baselines for the tasks of keyword summary generation and trending keywords selection.\footnote{Code and data are available at \href{https://github.com/pritomsaha/keyword-selection}{https://github.com/ pritomsaha/keyword-selection}}
\end{abstract}
\section{Introduction}
\label{sec:introduction}
\emph{Domain representative keywords} are the core knowledge of a \emph{target domain} of interest. A target domain can be a broad area of science like \textit{computer science (CS)} or its sub-field \textit{artificial intelligence (AI)}. 
Acquiring domain representative keywords benefits various natural language processing (NLP) tasks such as information summarization, organization, and extraction. For instance, acquiring a set of domain representative keywords is an important first step in organizing domain knowledge with a taxonomy of keywords \cite{zhang2018taxogen}. Moreover, tagging documents \cite{chen2017doctag2vec} with domain representative keywords helps to facilitate search or recommendation in a domain. For another example, summarizing a domain using its trending keywords for a specific time frame helps researchers get a snapshot of research trends or emerging areas of interest, e.g., new emerging security vulnerabilities.

\begin{figure}
\centering
\centerline{\includegraphics[width=\linewidth]{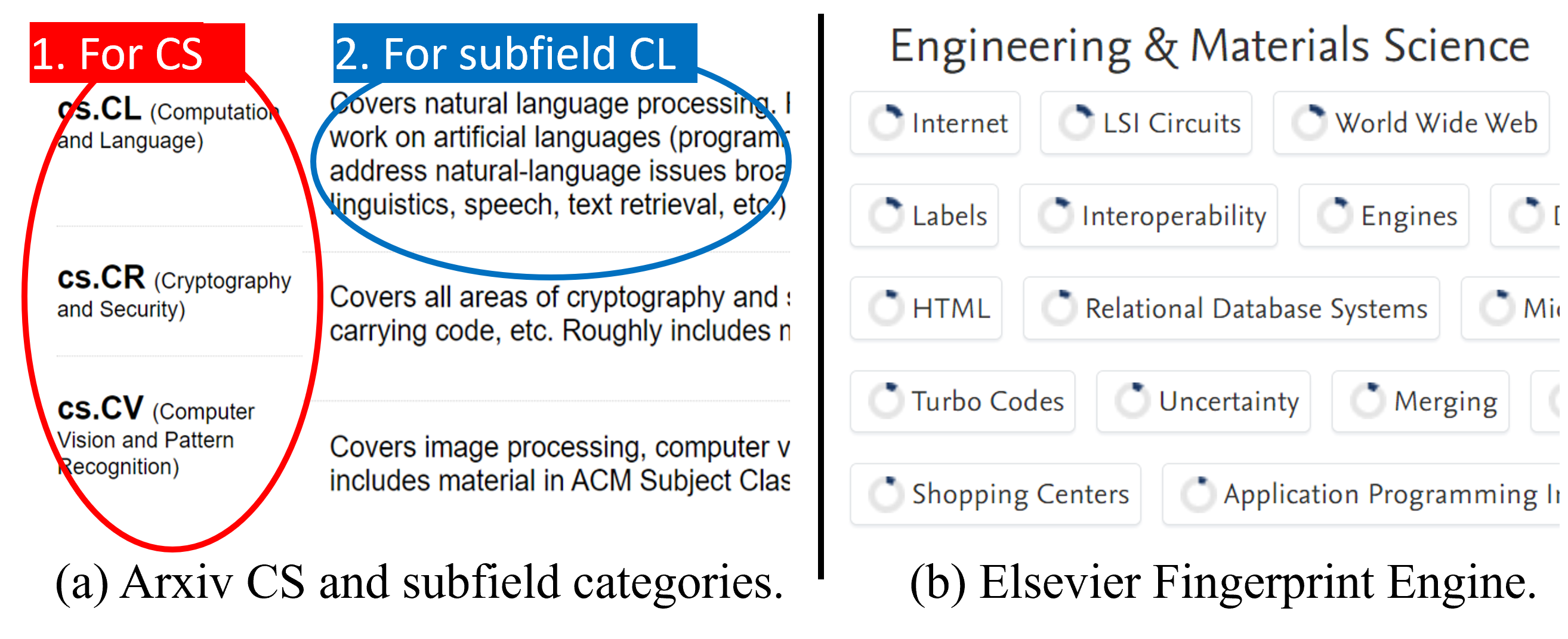}}
\caption{Screenshots of example applications.}
\label{fig:app}
\end{figure}
In reality, while representing a domain, the desired keywords often depend on a given \emph{context domain}. E.g., if we are interested in representing the CS domain with the context of general knowledge (all areas of knowledge), keywords like \textit{model}, \textit{data}, \textit{information} make sense in distinguishing CS from general knowledge. However, if the context is general science, those keywords are not \textit{distinctive} enough to distinguish CS from other areas like mathematics or physics. Instead, the keywords \textit{machine learning}, \textit{data mining}, \textit{deep learning} make more sense in this case. Therefore, it is important to contrast with a known context domain while representing a target domain, but most of the existing work ignored this. An application of this is shown in Fig. \ref{fig:app} (a) from the \textit{arxiv category taxonomy}\footnote{\href{https://arxiv.org/category_taxonomy}{https://arxiv.org/category\_taxonomy}}. We can see a shift of categories from general to more specific subcategories of research areas:  CS $\rightarrow$ CL. Knowing CS categories (partially shown in (1)) as the context, the keywords specified in (2) (i.e., \textit{speech}, \textit{text retrieval}) are more appropriate to represent CL than keywords overlapped with other CS categories. \looseness=-1

Moreover, the number of keywords that need to be selected depends on the nature of the applications. E.g., users are interested in a quick, high-level overview with fewer keywords while summarizing a particular domain. On the other hand, while building a controlled vocabulary representative of a certain domain, the number of keywords is naturally large. An application of the controlled vocabulary is illustrated in Figure \ref{fig:app} (b). It shows the fingerprint visualization of a CS researcher generated by Elsevier Fingerprint Engine\footnote{\href{https://www.elsevier.com/solutions/elsevier-fingerprint-engine}{https://www.elsevier.com/solutions/elsevier-fingerprint-engine}}, a system for research profiling. The researcher's profile is summarized using keywords in \textit{Engineering \& Material Science} domain. However, we can see that some non-representative keywords like \textit{labels} and \textit{merging} are used in the summarization. Therefore, having a representative controlled vocabulary for each domain will facilitate this application for expressively representing a researcher profile.

We thus propose the \textbf{problem} of \textit{domain representative keywords selection}. As input, we are given a set of candidate keywords, a target and a context domain represented by their corresponding corpora, and a size $k$. As output, we aim to select a subset consisting of $k$ keywords from the given candidate set such that the subset best represents the target domain contrasting with the context domain. Here, we assume that the candidate keywords are from the target domain and can be implicitly extracted from the given target domain corpus or externally given keywords for that domain.

From the above problem and discussion, we have identified that the solution for the problem needs to meet the following two \textbf{requirements}:
(1) the selected keywords should be distinctive to the target domain contrastive with a context domain;
(2) the selected keywords should represent the target domain as a whole within the specific size constraint.
None of the existing work satisfies all of them. 
Previously, research has been conducted on automatic keyword extraction \cite{hatty2017evaluating,meng2017deep, alzaidy2019bi,wang2020mining} and phrase mining \cite{liu2015mining,shang2018automated}. However, their main focus is to extract terms from single/multiple documents without considering whether the extracted terms are distinctive to a target domain contrastive with a context.
There is also some previous research \cite{liu2015mining,shang2018automated,lu2019concept,huang2021measuring} that tries to find fine-grained domain-specific keywords from the text. However, these approaches mostly rank keywords based on their specificity to a corpus (or domain) rather than selecting a predefined number of keywords with a global objective of representing the target domain. Therefore, in this work, we propose a solution to satisfy all the specified requirements.

The first \textbf{challenge} on fulfilling the requirements is \textit{contrasting the target and context domains}. Among candidate keywords, the distinctive keywords may have similar corpus statistics (i.e., frequency from target domain corpus) with many non-distinctive popular keywords. Therefore, simply filtering out highly frequent keywords may lose many distinctive keywords for a target domain. Instead, it is more intuitive to say that the keywords that frequently appear in both target and context corpora are often not distinctive keywords for the target domain. It inspires us to leverage the two-component mixture model (MM) \cite{zhai2001model} concept to generate the candidate keywords distribution contrasting with the context domain. 
As far as we know, this is the first work to utilize a mixture model mechanism for keywords selection. 
\looseness=-1

The second \textbf{challenge} is the \textit{representation under a size constraint}. If we simply select the top distinctive keywords based on the MM-generated distribution, we may end up with redundant keywords that may fall short in representing the target domain as a whole. Hence, it is more intuitive to consider selecting keywords with a domain representation objective. Therefore, we cast this as an optimization problem of selecting $k$ keywords that \textit{coarsen} the candidate distribution adapting the concept of \textit{statistical machine translation} \cite{brown1993mathematics} with the objective of minimizing the divergence between the initial and coarsened distributions of candidate keywords.

In summary, as our \textbf{contributions} in this paper, 
\textbf{firstly}, we propose a new problem formulation named \textit{domain representative keywords selection}. \textbf{Secondly}, we propose a framework for solving the problem consisting of two steps: (1) generating candidate keywords distribution using a two-component mixture model mechanism and (2) selecting a subset of keywords utilizing the generated distribution with an introduced optimization algorithm. \textbf{Thirdly}, we prove that our proposed optimization problem can be efficiently solved with a near-optimal approximation ratio. \textbf{Finally}, to validate the effectiveness of our approaches, we conduct extensive experiments on multiple domains for different tasks demonstrating the superiority of our framework against strongly designed baselines.
\section{Related Work}
\label{sec:related_work}
The problem of \textit{domain representative keywords selection} is related to the automatic keyword extraction (AKE) problem. AKE focuses on extracting or generating the most prominent keywords from single/multiple documents. Existing methods for AKE can be classified into two categories: supervised and unsupervised keyword extraction. Early supervised methods consider AKE as a binary classification problem \citep{DBLP:conf/dl/WittenPFGN99,turney2000learning} by learning a classifier from annotated documents to predict whether a candidate phrase is a keyword or not. 
Recently, \textit{deep learning} has been used for the supervised AKE. E.g., \cite{meng2017deep} uses an encoder-decoder-based framework to generate keywords where \cite{alzaidy2019bi} addresses AKE as a sequence labeling problem. Unsupervised AKE methods mostly apply graph-based ranking mechanisms utilizing semantic relatedness measure between keywords \cite{mihalcea2004textrank}. Besides, linguistic \cite{handler2016bag} and semantic \cite{bennani2018simple} approaches have also been used for unsupervised AKE. 

However, the main focus of the above studies is to describe single/multiple documents rather than domain-specific keywords extraction. To solve this problem, several researches \cite{liu2015mining,shang2018automated,lu2019concept,wang2020mining,huang2021measuring} have been conducted on domain-specific fine-grained keyword extraction. E.g., \cite{huang2021measuring} propose an algorithm for measuring the relevance of a keyword in a particular domain. However, this approach requires a user to provide some seed domain-relevant terms for supervising the algorithm. Moreover, the above approaches only consider ranking keywords based on their domain specificity (or relevance). None of them deals with the problem of domain representative keyword selection with a specific size constraint.

The mixture model used for generating keywords distribution in our approach is related to the research on probabilistic topic models \cite{hofmann2001unsupervised,blei2003latent} and comparative text mining \cite{sarawagi2003cross,zhai2004cross}. However, the difference between our approach and these studies is that rather than finding multiple latent topics or themes from a collection or multiple collections of documents, we model a target domain corpus as a distribution of unigram language model contrastive with a context model.
\section{Proposed Methodology}
\label{sec:proposed_method}
Our proposed framework consists of two steps: (1) generating distribution for the candidate keywords and (2) selecting a subset that best represents the target domain utilizing the generated distribution.
\subsection{Keywords Distribution Generation}
\label{sec:cand_dist}
To select keywords, how do we represent a target domain in contrast with a context one? 
One naive solution can be the frequency distribution of keywords in the target corpus. However, this distribution is biased towards common but possibly non-distinctive keywords (e.g., \textit{data}, \textit{method} and \textit{model} in CS), which may not differentiate the target (e.g., CS) from the context (e.g., Physics) domain. On the other hand, among candidate keywords, the distinctive keywords may have similar corpus statistics (i.e., frequency from target domain corpus) with many non-distinctive common keywords. Therefore, it is not easy to separate those desired target domain keywords from non-distinctive common keywords using simple statistics calculated from the target domain corpus. E.g., keyword \textit{algorithm} is more distinctive than \textit{method}, but both are popular keywords in CS domain. 
Therefore, simply filtering out highly frequent keywords may lose many distinctive keywords for a target domain.

To handle the above problem, we regard the target corpus as a mixture of two unigram language models. Specifically, the corpus is assumed to be generated from a mixture of two multinomial component models. One model is the known background model $\theta_B$ (computed from the context corpus), which models the non-distinctive common keywords in the target and context corpora. The other one is the target domain model ($\theta_D$) that needs to be estimated and concerned for prioritizing distinctive keywords in that domain. 

Formally, let $\mathcal{C}$ be the target domain corpus from which we are interested to find the keyword distribution, then the log-likelihood value (LLV) of generating $\mathcal{C}$ from this mixture model is
\resizebox{1.0\linewidth}{!}{
  \begin{minipage}{\linewidth}
    \begin{align}
    &\log p(\mathcal{C}|\theta_D) = \nonumber\\ & \sum_{t_i \in V} c(t_i,\mathcal{C}) \log [ (1 - \lambda) p(t_i|\theta_D) + \lambda p(t_i|\theta_B)]
    \label{eq:mixture_llh},
    \end{align}
  \end{minipage}
}
where $V$ is the candidate keywords set and $c(t_i, \mathcal{C})$ is the frequency of keyword $t_i$ in $\mathcal{C}$. $\lambda$ refers to the mixing weight of the $\theta_B$. In other words, $\lambda$ controls the amount of “background noise” in the corpus we want to be modeled by $\theta_B$. We assume $\theta_B$ and $\lambda$ to be known, and $\theta_D$ be estimated. Specifically, $\theta_B$ is the probability distribution calculated from the context domain corpus. 

In principle, we can estimate $\theta_D$ using any optimization methods. E.g., the Expectation-Maximization (EM) algorithm \cite{dempster1977maximum} is one of them and can be used to compute a maximum likelihood estimate with the following updating formulas:
\resizebox{1.0\linewidth}{!}{
  \begin{minipage}{\linewidth}
  \begin{align*}
    p^{(n)}(z=0|t_i) = \frac{(1-\lambda)p^{(n)}(t_i|\theta_D)}{(1-\lambda)p^{(n)}(t_i|\theta_D) + \lambda p^{(n)}(t_i|\theta_B)},\\
    p^{(n+1)}(t_i|\theta_D) = \frac{c(t_i, \mathcal{C})p^{(n)}(z=0|t_i)}{\sum_{t_j \in V} c(t_j, \mathcal{C})p^{(n)}(z=0|t_j)},
\end{align*}
\vspace{1mm}
  \end{minipage}
}
where $p(z=0|t_i)$ refers how likely $t_i$ is from $\theta_D$. The estimated $\{p(t_1|\theta_D)\cdots p(t_N|\theta_D)\}$ is used as candidate keywords distribution.

\subsection{Keyword Subset Selection}
\label{sec:subselection}
After acquiring a distribution of candidate keywords, we find a subset with a size $k$ to represent the target domain. One possible solution is to select top $k$ keywords based on the candidate distribution ($\theta_D$) generated by the \textit{mixture model} (MM). Hence, the keywords with high distinctiveness to the target domain contrasting with the context domain will be selected. However, one problem with this approach is that the selected keywords may fall short in representing the target domain by only selecting some redundant distinctive keywords. 

To solve the above problem, we view the subset selection as a \textit{distribution coarsening} problem. Specifically, we want to use a subset to estimate the candidate distribution (i.e., coarsened distribution). As defined in the previous section, a domain is a distribution of keywords (i.e., candidate distribution). Therefore, for a subset of keywords to represent the domain, the coarsened distribution by the subset should closely approximate the candidate distribution of that domain.

Formally, let $\mathcal{P} = \{p(t_i)\cdots p(t_N)\}$ be the candidate distribution, we compute a coarsened distribution $\tilde{\mathcal{P}} = \{\tilde{p}(t_i)\cdots \tilde{p}(t_N)\}$ by subset $S$ and $\tilde{p}(t_i)$ for each $t_i \in V$ is calculated as:
\resizebox{1.0\linewidth}{!}{
  \begin{minipage}{\linewidth}
    \begin{align}
    \tilde{p}(t_i) = \sum_{t_j \in S} p(t_i|t_j) p(t_j)
    \label{eq:estpi},
    \end{align}
  \end{minipage}
}
where $ p(t_i|t_j)$ refers to the probability of \textit{semantically translating} $t_j$ into $t_i$. This idea of estimating the probability of each keyword from candidates by a subset is adapted from the statistical machine translation from the same language used in information retrieval \cite{DBLP:conf/sigir/BergerL99}. 

Now to find the subset, we introduce an optimization problem with objective of selecting a subset ($S$) with size $k$ from candidates ($V$) that minimizes the difference between the LLV of generating $\mathcal{C}$ by $\mathcal{P}$ and $\tilde{\mathcal{P}}$, respectively. We know that the LLV of generating $\mathcal{C}$ by $\mathcal{P}$ is $\log p(\mathcal{C}) = \sum_{i \in V} c(t_i,\mathcal{C}) \log p(t_i)$ where $c(t_i, \mathcal{C})$ is the frequency of $t_i$ in $\mathcal{C}$. Similarly, the LLV of generating $\mathcal{C}$ from $\tilde{\mathcal{P}}$ is $\log \tilde{p}(\mathcal{C})$. Hence, given $|S|=k$, our optimization objective is:
\resizebox{1.0\linewidth}{!}{
  \begin{minipage}{\linewidth}
    \begin{align}
    S &= \arg \min_{S \subseteq V} \|\log p(\mathcal{C}) - \log \tilde{p}(\mathcal{C})\| \nonumber\\
    &= \arg \min_{S \subseteq V} \|\sum_{t_i \in V} c(t_i,\mathcal{C}) \log \frac{p(t_i)}{\tilde{p}(t_i)}\| \nonumber \\
    &= \arg \min_{S \subseteq V} \|\sum_{t_i \in V} p(t_i) \log \frac{p(t_i)}{\tilde{p}(t_i)}\| \nonumber \\
    &= \arg \min_{S \subseteq V} D_{KL}(\mathcal{P} \| \tilde{\mathcal{P}})
    = \arg \min_{S \subseteq V} \phi(S)
    \label{eq:objective},
    \end{align}
    \vspace{0.5mm}
  \end{minipage}
}
where $\phi(S)$ is our objective function, and $D_{KL}(\mathcal{P}\|\tilde{\mathcal{P}})$ is Kullback–Leibler (KL) divergence \cite{kullback1951information} between $\mathcal{P}$ and $\tilde{\mathcal{P}}$.

From \eqref{eq:estpi}, one obvious question is how to calculate $p(t_i|t_j)$. For this, we use \textit{mutual information} (MI) to estimate $p(t_i|t_j)$ inspired from \cite{karimzadehgan2010estimation} where MI is used to estimate a similar model for information retrieval. MI is a good measure to judge relatedness between two terms. In our model, for any two terms $t_i$ and $t_j$, we first compute MI ($I(t_i;t_j)$) between them and normalize it into a probability as below:
\resizebox{1.0\linewidth}{!}{
\begin{minipage}{\linewidth}
\begin{align}
    I(t_i;t_j) = \sum_{b_{i}, b_{j}} p(b_{i}, b_{j}) \log \frac{p(b_{i}, b_{j})}{p(b_{i}) p(b_{j})}, 
    \label{eq:mi}
    \\
    p(t_i|t_j) \approx p_{MI}(t_i|t_j) = \frac{I(t_i;t_j)}{\sum_{t_j^{\prime} \in V} I(t_i;t_j^{\prime})}
    \label{eq:pmi},
\end{align}
\vspace{0.5mm}
\end{minipage}
}
where $b_{i}$ is a binary variable indicating the presence/absence of $t_i$. E.g., $p(b_{i}=1)$ indicates the ratio of documents containing $t_i$ and $p(b_{i}=1, b_{j}=1)$ indicates the ratio of documents where both $t_i$ and $t_j$ co-occur. Here, $p_{MI}(t_i|t_j)$ gives us the probability of how $t_j$ relates to $t_i$; intuitively, this probability would be higher when these two terms frequently co-occur in the same document in the target corpus.

{\flushleft \textbf{Optimization.}}
We are interested in finding a subset $S$ with size $k$ from $V$ such that $\phi(S)$ is minimized, i.e., $\arg \min_{S\in V} \phi(S) \text{ s.t. } |S| = k$. This is referred as the \textit{cardinality-constrained optimization} and proven to be NP-hard \cite{feige1998threshold}. However, if the objective function $\phi(S)$ is monotone and submodular, a simple greedy algorithm is guaranteed to obtain an approximation of $1 - \frac{1}{e}$. We call a non-negative real valued function $F$ (to be maximized) \textit{submodular} if it has the property of \textit{diminishing returns} that is $F(X\cup \{v\}) - F(\{v\}) \geq F(Y\cup \{v\}) - F(\{v\})$ for all $v\in V$ and $X\subseteq Y \subseteq V$. Moreover, $F$ is said to be \textit{monotone} if $F(X) \leq F(Y)$ for all $X\subseteq Y$.
\begin{theorem}
For minimizing the objective function $\phi(\cdot)$, a simple greedy algorithm obtains an approximation guarantee of $1 - \frac{1}{e}$.
\label{thm:greedy_opt}
\end{theorem}
\begin{proof}

The proof can be found in Appendix \ref{sec:appendix_proof}.
\end{proof}

So, as per Theorem \ref{thm:greedy_opt}, we can obtain a near optimal solution using a simple greedy algorithm. Initially, we have $S = \emptyset$, then iteratively update $S = S \cup  \arg \max_{t\in V\setminus S} \mathcal{G}(t|S)$ until $|S| = k$ where $\mathcal{G}(t|S) = \phi(S) - \phi(S \cup {t})$ is the gain of adding a new term $t$ to $S$. Thanks to the submodularity property of $\phi(\cdot)$, this simple greedy algorithm can further be accelerated by lazy greedy algorithm \cite{minoux1978accelerated}. More specifically, instead of recomputing $\mathcal{G}(t_i|S), \forall t_i\in V$ in every step, we use a priority queue of sorted gains $g(t_i), \forall t_i\in V$. Starting with $g(t_i) = - \phi(\{t_i\}), \forall t_i \in V $, the algorithm adds a term $t_i$ to $S$ if $g(t_i) \geq \mathcal{G}(t_i|S)$, otherwise we update $g(t_i) \text{ to } \mathcal{G}(t_i|S)$ and resort the priority queue. This largely improves the efficiency of the algorithm.   
\section{Experiments}
\label{sec:experiments}
This section evaluates our models from different perspectives: (1) the ability to select representative summary keywords for a target domain; (2) the performance for trending keywords selection task in a domain for different time frames.

\subsection{Experiment Setup}
{\flushleft \textbf{Datasets.}} In our experiments, to test the generality of the proposed approaches, we use two document collections from two domains for constructing target and context corpora for each of the domains. One is abstracts collections from the \textit{arxiv} repository (version 47)\footnote{\href{https://www.kaggle.com/Cornell-University/arxiv}{https://www.kaggle.com/Cornell-University/arxiv}}, and the other is a collection of newsgroup documents\footnote{\href{http://qwone.com/~jason/20Newsgroups}{http://qwone.com/~jason/20Newsgroups}}.

{\flushleft \textbf{Candidate Keywords.}}
In our experiments, we use different sets of candidate keywords. For the CS domain, we collected keywords from two external sources named Springer and Aminer \cite{tang2008arnetminer}. The Springer CS keyword list is collected through web scraping from Springer\footnote{\href{https://www.springer.com/gp}{https://www.springer.com/gp}} and trimmed to 83K based on frequency $\geq$ 5. The Aminer keyword list is the collection of keywords assigned by authors in CS research papers, and there are approximately 50K keywords in this list. Alongside keywords from external sources, we also created candidate sets extracted from concerning corpus using AutoPhrase \cite{shang2018automated} tool. All the candidate keywords are lemmatized, and several filtering rules are used. For instance, keywords containing only letters, numbers, hyphens are used; stop and single-letter words are removed.

{\flushleft \textbf{Baselines.}} We compare our models with the following four baseline keyword selection algorithms.
\begin{itemize}[nolistsep,leftmargin=*]
\item \textbf{Relative Frequency (RF):}
Since a keyword is likely to be domain representative when it frequently appears in a domain corpus, we consider a simple approach that selects the top $k$ frequent keywords based on the relative frequency calculated from the target corpus.
    
\item \textbf{Log-odds (LO):} 
We adapted a method \cite{monroe2008fightin} for keyword selection which was introduced to compare words used by two political parties. Recently, \cite{hughes2020detecting} used this method for detecting trending terms in \textit{Cybersecurity} forum discussion. In this baseline, we adapt this method to model keywords as a function of a particular domain or time to compute the likelihood of keywords in that domain or time as log-likelihoods (``log-odds'').

\item \textbf{Page Rank (PR):} 
This baseline is a graph-based keyword selection method using PageRank \cite{mihalcea2004textrank}. We build the graph of candidate keywords where each edge weight denotes how closely two keywords are related.
    
\item \textbf{Facility Location (FL) Function:} 
Facility location function is a representation based subset selection measure \cite{mirchandani1990discrete} used for finding a representative subset of items. Recently, this measure is used for training-data subset selection \cite{kaushal2019learning}. In this paper, we adapt this measure as a baseline for selecting subset from candidate keywords set. Specifically, denoting $rel(t_i, t_j)$ as the relatedness of two keywords $t_i$ and $t_j$, the objective is to select a subset $S \in V$ that maximizes FL function $f(S) = \sum_{t_i\in V} \max_{t_j\in S} rel(t_i, t_j)$. 
\end{itemize}

{\flushleft \textbf{Proposed Models.}}
We have the following three variants of our proposed framework.
\begin{itemize}[nolistsep,leftmargin=*]
\item \textbf{KL divergence + RF ($\mathbf{KL_{rf}}$):}
This model is a simple version of our proposed objective function $D_{KL}(\mathcal{P}\|\tilde{\mathcal{P}})$ defined in \eqref{eq:objective}. In this model, $\mathcal{P}$ is the relative frequency distribution calculated from the target corpus and $\tilde{\mathcal{P}}$ is coarsened distribution defined in \eqref{eq:estpi}.  

\item \textbf{Mixture Model (MM):}
In this proposed model, keywords are ranked based on the estimated distribution for the target domain contrasting with a context domain using the mixture model defined in Section \ref{sec:cand_dist}. Based on the distribution, the top $k$ keywords are selected.   

\item \textbf{KL Divergence + MM ($\mathbf{KL_{mm}}$):}
This proposed model is similar to $\mathbf{KL_{rf}}$. In $\mathbf{KL_{mm}}$, instead of using relative frequency, the mixture model estimated keyword distribution is used as $\mathcal{P}$ in $D_{KL}(\mathcal{P}\|\tilde{\mathcal{P}})$.
\end{itemize}

{\flushleft \textbf{Implementation Details.}}
There are some parameters both in baselines and the proposed models we have to set. E.g., the mixing weight $\lambda$ for the background model in the mixture model is set to two different values based on the specificity of the target domains. Particularly, when we set $\lambda$ to a small value, the model favors frequent non-informative terms (i.e., domain-specific stop words). Therefore, the larger values are set for $\lambda$. In our experiments, for a broad domain like CS, we set $\lambda$ to 0.9, and for more specific domains (i.e., AI and subtopics in newsgroup), we set $\lambda$ to 0.99. The reason for these two different values of $\lambda$ is that more specific domains demand larger $\lambda$ for selecting distinctive keywords. For optimizing MM, we use Expectation-Maximization (EM) algorithm \cite{dempster1977maximum}. Since EM does not guarantee the global maxima, in our experiment, we run the algorithm multiple times with random initialization, and the one with the best MLE is chosen to reduce the chance of getting local maxima. As we use mutual information (MI) based on document co-occurrence statistics in our model (defined in \eqref{eq:pmi}), for the fair comparison, in the baseline FL, we also use MI between two keywords $t_i$ and $t_j$ to encode the relatedness between them (i.e., $rel(t_i, t_j)$). Similarly, MI is used for computing edge weight in the PR method.

\subsection{Experiment Results}
\subsubsection{Summary Keywords Selection}
We conduct both quantitative and qualitative studies to evaluate the ability of proposed models to select domain representative summary keywords. For this purpose, we use the abstracts from the arxiv under CS categories as the target corpus. The context corpus is composed of all abstracts in the arxiv repository.
\begin{table}[!t]
\centering
\resizebox{1.0\linewidth}{!}{%
\begin{tabular}{@{}p{1em}PPPPPPP@{}}
\toprule
$k$   & \multicolumn{1}{c}{RF}      & \multicolumn{1}{c}{LO}      & \multicolumn{1}{c}{PR}      & \multicolumn{1}{c}{FL}      &  \multicolumn{1}{c}{$\mathbf{KL_{rf}}$} & \multicolumn{1}{c}{MM}               & \multicolumn{1}{c}{$\mathbf{KL_{mm}}$} \\ \midrule
10  & 1.0651  & 1.1001  & 1.1035  & 1.0722  & 1.0651                                                 & 2.0981           & \textbf{2.1212}                                            \\
20  & 1.1440  & 3.2476  & 2.2682  & 1.1345  & 1.1451                                                 & 4.2626           & \textbf{4.2972}                                            \\
30  & 2.2134  & 4.3965  & 3.3607  & 3.2682  & 3.2875                                                 & \textbf{4.4321}  & 4.4169                                                     \\
40  & 3.3273  & 4.4929  & 4.4902  & 3.3515  & 3.3660                                                 & 4.6000           & \textbf{4.6018}                                            \\
50  & 3.4530  & 4.6896  & 4.5734  & 3.4505  & 3.4496                                                 & \textbf{5.6826}  & \textbf{5.6826}                                            \\
100 & 4.7812  & 8.2382  & 7.0399  & 4.7626  & 4.8761                                                 & 8.2708           & \textbf{8.2824}                                            \\
200 & 9.7166  & 11.1047 & 8.9403  & 9.5908  & 8.7045                                                 & 11.1082          & \textbf{12.0233}                                           \\
500 & 18.902  & 19.1719 & 18.0464 & 16.6171 & 17.7441                                                & \textbf{19.3221} & 19.2353                                            \\ \bottomrule
\end{tabular}}
\caption{Category correspondence results}
\label{tab:domain_summ_nmi}
\end{table}
{\flushleft \textbf{Quantitative Evaluation.}}
We create keyword summaries for the CS domain with varying sizes ($k$) for quantitative evaluation. We collected $52$ known category keywords from arxiv categories as CS representative ground keywords to evaluate the ability of selected $k$ summary keywords to represent the target domain when $k$ varies. The correspondence between $k$ selected keywords $S=\{t_1\cdots t_k\}$ and $m$ category keywords $C=\{c_1\cdots c_m\}$, $CC(S,C)$ is calculated as the summation of the pairwise normalized mutual information (NMI) \cite{bouma2009normalized} between $S$ and $C$ i.e., $CC(S,C) = \sum_{i,j}\frac{I(t_i; c_j)}{H(t_i;c_j)}$ where $I(t_i; c_j)$ is calculated following formula from \eqref{eq:mi} and $H(t_i;c_j) = -\sum_{b_{i}, b_{j}} p(b_{i}, b_{j}) \log p(b_{i}, b_{j})$ is the joint entropy of $t_i$ and $c_j$.

From the results on Table \ref{tab:domain_summ_nmi}, using AutoPhrase extracted candidate keywords, we can see that even though no supervision is used, our methods $\mathbf{KL_{mm}}$ and MM select keywords that best correspond with the known categories outperforming all the baselines (similar results from two more candidate sets are shown in Appendix \ref{sec:additional_quancs}). We observe that there is a good improvement of result from MM to $\mathbf{KL_{mm}}$. However, this is not true for $\mathbf{KL_{rf}}$ and the RF baseline. The reason is that the relative frequencies from the target corpus favor the non-distinctive common keywords (e.g., \textit{model} and \textit{method}). As described in Section \ref{sec:subselection}, $\mathbf{KL_{rf}}$ tries to select the subset of keywords that best estimate the original candidate distribution. Hence, it also favors those common keywords to attain the nearest estimation of the original distribution. 

On the other hand, the MM-generated distribution assigns larger probabilities to distinctive keywords of the target domain, contrasting with the context domain. Therefore, selecting a keyword subset by $\mathbf{KL_{mm}}$ with close estimation of the MM generated distribution also favors distinctive keywords with the domain representative objective. Furthermore, one interesting observation is that when $k$ is smaller, the selected keywords by $\mathbf{KL_{mm}}$ tend to summarize the domain better than that of MM. The primary reason for this is that $\mathbf{KL_{mm}}$ prefers to select more non-redundant keywords than MM while $k$ is smaller, which we later discuss from Table \ref{tab:cs_summary_ap}.
\begin{table}[!t]
\scriptsize
\resizebox{1.0\linewidth}{!}{%
\begin{tabular}{p{0.06\linewidth}p{0.74\linewidth}}
\toprule
Models & \multicolumn{1}{c}{Selected 20 keywords in CS} \\ \midrule
RF                                            & paper, model, \textbf{algorithm}, datum, result, information, \textbf{graph}, state, high, art, single,   order, human, research, general, design, \textbf{deep learning}, \textbf{semantic},   knowledge, \textbf{neural network}                                                                               \\ \midrule
LO                                            & \textbf{algorithm}, art,   information, \textbf{semantic}, \textbf{graph}, human, \textbf{deep learning}, paper, datum, \textbf{neural network}, \textbf{machine learning}, real world,   research, video, \textbf{\textbf{robot}}, communication, language, \textbf{security}, \textbf{architecture}, knowledge                                            \\ \midrule
PR                                            & polynomial,   research, channel, paper, energy, \textbf{graph}, datum, model, experimental, information, \textbf{machine learning}, \textbf{software}, binary, english,   propose method, function, acoustic, upper, solution, algebraic                                                       \\ \midrule
FL                                            & art, paper, datum, \textbf{algorithm}, model, result, high, information, \textbf{graph}, channel, research, order, single, human, general, \textbf{deep learning}, design, experimental, solution,   knowledge                                                                                  \\ \midrule
{$\mathbf{KL_{rf}}$} & model, \textbf{algorithm},   paper, datum, state, \textbf{graph}, result, information, high, art, human, research, design, single, \textbf{semantic}, order, \textbf{deep learning}, energy, general, \textbf{neural network}                                                                                    \\ \midrule
MM                                            & \textbf{algorithm}, art,   \textbf{semantic}, \textbf{deep learning}, human, \textbf{neural network}, \textbf{convolutional neural network}, \textbf{machine learning}, real world, video,   information, \textbf{robot},   research, language, communication, \textbf{security},   \textbf{architecture}, \textbf{privacy}, \textbf{deep neural network}, \textbf{label} \\ \midrule
{$\mathbf{KL_{mm}}$} & \textbf{algorithm}, art,   \textbf{semantic}, \textbf{deep learning}, human, \textbf{security}, \textbf{neural network}, real world, \textbf{convolutional neural network}, communication, \textbf{machine learning}, \textbf{robot},   language, video, research, \textbf{privacy}, \textbf{label}, information, \textbf{software}, \textbf{architecture}  \\
\bottomrule
\end{tabular}}
{\raggedright \centering Keywords distinctive to the CS domain are \textbf{highlighted} (annotated by authors). \par}
\caption{Summary keywords in CS Domain}
\label{tab:cs_summary_ap}
\end{table}
\begin{table}[!t]
\centering
\resizebox{0.8\linewidth}{!}{%
\begin{tabular}{@{}lccc@{}}
\toprule
 Models                                                 & 2000-2009       & 2010-2019       & 2020-2021       \\ \midrule
RF                                                         & 0.6289          & 0.6640           & 0.6493          \\
LO                                                         & 0.6813          & 0.7199          & 0.7238          \\
PR                                                         & 0.6626          & 0.6970          & 0.6826          \\
FL                                                         & 0.6172          & 0.6848          & 0.6528          \\ \midrule
{$\mathbf{KL_{rf}}$}     & 0.6282          & 0.6792          & 0.6516          \\
MM                                                         & \textbf{0.6908} & 0.7331          & 0.7898          \\
{$\mathbf{KL_{mm}}$} & 0.6898          & \textbf{0.7763} & \textbf{0.7944} \\ \bottomrule
\end{tabular}}
\caption{Results using trending ground truth keywords}
\label{tab:time_summ_sim}
\end{table}
\begin{table}[!t]
\centering
\resizebox{0.8\linewidth}{!}{%
\begin{tabular}{@{}lccc@{}}
\toprule
 Models                                             & 2000-2009 & 2010-2019 & 2020-2021 \\ \midrule
RF                                            & 0.3593    & 0.3195    & 0.323     \\
LO                                            & 0.3956    & 0.3326    & 0.4043    \\
PR                                            & 0.3705    & 0.3211    & 0.3641    \\
FL                     & 0.3591    & 0.3217    & 0.3336    \\\midrule
{$\mathbf{KL_{rf}}$} & 0.3583    & 0.3189    & 0.3239    \\ 
MM                     & 0.4104    & 0.3468    & 0.5145    \\
{$\mathbf{KL_{mm}}$} & \textbf{0.4165}    &\textbf{ 0.3523}    & \textbf{0.5215}    \\ \bottomrule
\end{tabular}}
\caption{Results generated using Google Trends}
\label{tab:time_summ_gt}
\end{table}

{\flushleft \textbf{Qualitative Evaluation.}}
For the qualitative evaluation, we show the summary keywords selected by different algorithms in the CS domain from AutoPhrase extracted candidate keywords in Table \ref{tab:cs_summary_ap} (simmilar additional results are shown in Appendix \ref{sec:additional_cs}). This study aims to observe the difference between the proposed models and baselines in selecting summary keywords. We can see that our models (MM and $\mathbf{KL_{mm}}$) outperform all the baselines by selecting the most number of CS representative keywords. We also observe that the LO baseline method also selects a comparable amount of distinctive keywords. The reason is its use of a contrastive method like MM for selecting keywords for a particular corpus compared to a context corpus. 

However, our models MM and $\mathbf{KL_{mm}}$ tend to select more representative keywords than the LO method. 
E.g., we can see that our methods select keywords like \textit{privacy}, \textit{software} and \textit{convolutional neural network} instead of keywords that LO selects like \textit{graph} and \textit{paper}, \textit{data}. Another observation is that the keywords selected by PR are mostly those keywords (i.e., \textit{experimental}, \textit{data} and \textit{function}) that have a broad association with other words. However, these keywords as an unit do not convey much information about the domain. 


\begin{table*}[!hbt]
\scriptsize
\centering
\resizebox{1.0\linewidth}{!}{%
\begin{tabular}{p{0.04\textwidth}p{0.35\textwidth}p{0.35\textwidth}p{0.35\textwidth}}
\toprule
                                                          & \multicolumn{1}{c}{2000-2009}                                                                                                                                                                                                                                                                                                             & \multicolumn{1}{c}{2010-2019}                                                                                                                                                                                                                                                                                                                             & \multicolumn{1}{c}{2020-2021}                                                                                                                                                                                                                                                                                    \\ \midrule
                                                          
RF                                                         & paper, problem, \textbf{algorithm}, model, method, approach, \textbf{system},   information, result, datum, set, application, number, user, word,   performance, language, order, time, case                                                                                                                                                                & model, method, paper, approach, image, problem, datum, task,   algorithm, dataset, performance, network, result, feature, system, training,   application, work, number, object                                                                                                                                                       & model, method, task, datum, approach, dataset, image, paper,   performance, problem, training, algorithm, network, feature, result, system,   work, \textbf{application}, \textbf{deep learning}, experiment                                                                                                                                                                                    \\ \midrule

LO                                                         & \textbf{logic program}, rule, manipulator, \textbf{genetic algorithm},   workspace, \textbf{parallel manipulator}, \textbf{logic programming}, document, grammar, stable   model, \textbf{artificial immune system}, logic, word, web site, \textbf{answer set}, global   constraint, \textbf{machining}, fitness, belief, evolvability                                                                  & image, dataset, method, feature, task, \textbf{convolutional neural network}, object, training, \textbf{deep learning}, classification, classifier, \textbf{deep   neural network}, \textbf{neural network}, robot, model, video, \textbf{recurrent neural   network}, word, \textbf{segmentation}, representation                                                                        & model, dataset, task, training, \textbf{transformer}, image, \textbf{deep   learning}, \textbf{neural network}, prediction, label, \textbf{federated learning}, learning,   method, \textbf{machine learning}, \textbf{language model}, explanation, experiment, \textbf{covid 19},   \textbf{reinforcement learning}, feature                                                           \\\midrule
PR                                                         & problem, \textbf{algorithm}, paper, user, datum, model, word, method,   image, information, approach, \textbf{system}, constraint, set, solution, performance,   application, document, result, rule                                                                                                                                                        & image, algorithm, user, robot, object, network, word, model,   dataset, agent, environment, datum, task, video, method, training, \textbf{language},   system, policy, \textbf{segmentation}                                                                                                                                                            & image, robot, model, object, algorithm, dataset, task, agent,   environment, user, datum, policy, language, graph, \textbf{reinforcement learning},   network, method, video, training, \textbf{deep learning}                                                                                                                    \\ \midrule

FL                                                         & paper, problem, \textbf{algorithm}, model, method, approach, \textbf{system},   result, information, set, application, datum, number, user, word, order,   performance, case, image, time                                                                                                                                                                   & image, model, method, paper, problem, approach, datum,   algorithm, task, network, performance, dataset, result, user, application,   work, feature, system, training, number                                                                                                                                                         & image, model, method, paper, task, datum, approach, dataset,   performance, problem, training, algorithm, work, result, network, experiment,   \textbf{application}, system, feature, \textbf{deep learning}                                                                                                                       \\ \midrule

{$\mathbf{KL_{mm}}$}     & problem, paper, \textbf{algorithm}, method, model, \textbf{system}, approach,   information, datum, word, set, result, user, application, \textbf{agent}, number,   network, performance, \textbf{language}, order                                                                                                                                                            & image, model, method, algorithm, datum, paper, task, network,   problem, approach, dataset, system, user, feature, performance, training,   object, application, result, information                                                                                                                                                  & model, method, image, task, datum, dataset, problem, network,   approach, paper, training, algorithm, system, performance, feature, object,   \textbf{application}, user, \textbf{deep learning}, result                                                                                                     \\ \midrule

MM                                                         & \textbf{logic program}, manipulator, \textbf{genetic algorithm}, workspace,   \textbf{parallel manipulator}, \textbf{logic programming}, grammar, \textbf{stable model}, \textbf{artificial   immune system}, web site, \textbf{answer set}, \textbf{global constraint}, \textbf{machining}, fitness,   \textbf{evolvability}, \textbf{radial distortion}, \textbf{soft constraint}, \textbf{nonmonotonic reasoning},   \textbf{stable model semantic}, \textbf{belief revision} & image, \textbf{convolutional neural network}, \textbf{recurrent neural network},   classifier, \textbf{deep convolutional neural network}, \textbf{deep network}, \textbf{cnn}, \textbf{computer   vision}, \textbf{lstm}, \textbf{deep neural network}, \textbf{bayesian network}, \textbf{rnn}, \textbf{word embedding},   \textbf{svm}, \textbf{segmentation}, \textbf{convolutional network}, \textbf{descriptor}, \textbf{neural machine   translation}, recognition, sentence   & \textbf{transformer}, training, \textbf{federated learning}, \textbf{language model},   \textbf{covid 19}, \textbf{graph neural network}, dataset, explanation, \textbf{deep learning}, \textbf{pre   training}, \textbf{adversarial attack}, \textbf{fine tuning}, \textbf{meta learning}, \textbf{deep learning}   model, lidar, \textbf{self attention}, point cloud, \textbf{reinforcement learning}, \textbf{bert},   label               \\ \midrule

{$\mathbf{KL_{mm}}$} & \textbf{
logic program}, workspace, \textbf{genetic algorithm}, grammar,   manipulator, \textbf{logic programming}, web site, \textbf{global constraint}, \textbf{artificial   immune system}, \textbf{evolvability}, \textbf{parallel manipulator}, synonym, \textbf{stable model},   \textbf{som}, \textbf{belief revision}, \textbf{unification}, \textbf{soft constraint}, \textbf{language resource}, fitness,   \textbf{wordnet}                                          & image, \textbf{convolutional neural network}, \textbf{recurrent neural network},   classifier, \textbf{deep network}, \textbf{deep convolutional neural network}, \textbf{bayesian   network}, \textbf{word embedding}, \textbf{computer vision}, \textbf{descriptor}, \textbf{svm}, \textbf{crf}, \textbf{lstm}, \textbf{neural   machine translation}, dictionary, \textbf{convolutional network}, \textbf{deep neural network},   \textbf{recognition}, \textbf{cnn}, \textbf{segmentation} & \textbf{transformer}, training, explanation, \textbf{language model}, \textbf{covid 19},   \textbf{federated learning}, \textbf{graph neural network}, dataset, \textbf{pre training}, lidar, \textbf{deep   learning}, \textbf{adversarial attack}, label, \textbf{meta learning}, \textbf{knowledge distillation},   \textbf{fine tuning}, \textbf{deep learning model}, \textbf{latent space}, \textbf{datum augmentation}, \textbf{target   domain} \\ \bottomrule
\end{tabular}}
{\raggedright \centering Keywords representative of its corresponding time frame are \textbf{highlighted} (annotated by authors). \par}
\caption{Keyword summaries (top 20 keywords) of three different time frames in AI domain}
\label{tab:time_summ_full}
\end{table*}

Now to see the difference between our models MM and $\mathbf{KL_{mm}}$, we see the difference between their selected keywords. As stated before, we can see $\mathbf{KL_{mm}}$ prefers non-redundant keywords than MM. E.g, $\mathbf{KL_{mm}}$, instead of selecting \textit{deep neural network} as it already selects keywords like \textit{neural network} and \textit{deep learning}, it selects a different keyword \textit{software} where MM prefers redundant keyword \textit{deep neural network}. Therefore, while the only requirement is to rank keywords based on their distinctiveness for a target domain contrastive with a context domain, MM is more practical to use. On the other hand, if the objective is also selecting diverse representative keywords, $\mathbf{KL_{mm}}$ is preferable. See Appendix \ref{sec:newsgroup_result} for more qualitative study using newsgroup dataset. 

\subsubsection{Trending Keywords Selection}
As an important application of our problem, we evaluate the performance of proposed approaches for \textit{trending keywords selection} in the AI domain. This study conducts quantitative and qualitative evaluations considering three different time frames: 2000-2009, 2010-2019, and 2020-2021. For this purpose, we compose a corpus representative of each of the specified time frames by collecting abstracts from the Arxiv repository under AI-related categories: cs.AI, cs.CL, cs.CV, cs.IR, cs.LG, cs.NE and cs.RO. The entire dataset under all CS categories is used for the context corpus.

{\flushleft \textbf{Quantitative Evaluation.}}
Since there is no ground truth trending keywords available for the AI domain, it is not easy to quantitatively evaluate the selected ones for a specific time. Instead, we have created three ground truth sets by collecting related keywords from topic areas used in the call for papers (CFP) of an AI conference called AAAI\footnote{\href{https://www.aaai.org/}{https://www.aaai.org/}} over the three specified time frames. However, the topics that appear in the CFP are not necessarily trending topics, and many topics appear throughout all the time frames. For this, we collect only the changing topics from a time frame to another. Further, to expand the ground truth sets, we also add keywords related to the collected topics. E.g., \textit{word embedding} was a popular keyword in NLP during the 2010s, and one related of this is \textit{word2vec}.

{\flushleft \uline{Evaluation using Ground Truths}.}
For evaluation, we compute the selected keywords' ability to cover the ground truth keywords using a \textit{representativeness} measure. Formally, similar to \cite{kaushal2019learning}, say $s_{ij}$ denotes the similarity between two keywords $t_i$ and $t_j$, $R(S)=\frac{1}{|\mathcal{G}|} \sum_{t_i \in \mathcal{G}} \max_{t_j \in }s_{ij}$ is used as the \textit{representativeness} score of selected keyword set $S$ to represent the ground truth set $\mathcal{G}$. For $s_{ij}$, we compute the cosine similarity between vector representation of $t_i$ and $t_j$. The vector for each keyword is the concatenation of two word-vectors; one is word2vec (300d) \cite{mikolov2013efficient} learned from corresponding corpus, and the other is the compositional GloVe embedding \cite{pennington2014glove} (element-wise addition of the pre-trained 300d word embeddings). The reason for using pre-trained word vectors is that many keywords in ground truth sets do not appear in the corresponding corpus, and thus vectors cannot be learned from that corpus. Table \ref{tab:time_summ_sim} shows the detailed results over three time frames. We can see that our proposed model $\mathbf{KL_{mm}}$ outperforms the other methods with large margins followed by MM. 

{\flushleft \uline{Evaluation using Google Trends}.}
Alongside using ground truths, we also design a quantitative evaluation measure (shown in Table \ref{tab:time_summ_gt}) using Google Trends (GT) API\footnote{\href{https://github.com/GeneralMills/pytrends}{https://github.com/GeneralMills/pytrends}}.  GT\footnote{\href{https://trends.google.com}{https://trends.google.com}} awards a score for a term called \textit{interest over time} that expresses the term's popularity over a specified time range.  Since GT does not have data before 2004, we have to use data from 2004 till 2009 for the 2000-2009 time frame. As our three specified time frames are not equal, we first take the average of provided interest scores for each keyword in each time frame to make the score comparable across different time frames. 
Then, we calculate the probability of each term's interest over three specified time frames. Finally, the average of computed probability scores of 50 selected terms is calculated for each method. This score represents the average probability of selected terms to be trending in each time frame. From Table \ref{tab:time_summ_gt}, we can see that our method $\mathbf{KL_{mm}}$ achieves the best score over others, followed by comparable results from MM. It indicates that our solutions are more appropriate in finding trending keywords for a specified time frame.

{\flushleft \textbf{Qualitative Evaluation.}}
We qualitatively evaluate the performance of different algorithms by directly comparing their selected keywords in each time frame from Table \ref{tab:time_summ_full}. We can see PR selects keywords that are either CS stop words or the keywords that are not distinctive for a perspective time frame compared to others (similar results by RF, FL, $\mathbf{KL_{rf}}$). Because PR primarily depends on the popularity of a keyword and some keywords always appear frequently in any time frames (e.g., \textit{task}, \textit{dataset}, \textit{model}, \textit{etc} ). Here, the LO again provides comparable results. E.g., similar to our methods MM and $\mathbf{KL_{mm}}$, LO also can select very relevant trending keywords during the 2020s like \textit{covid 19}. However, while selecting trending keywords, the LO also tends to select many domain-specific stop words overlapped over different time frames (e.g., \textit{method, task, model}). As discussed before, the reason is that LO does not have the objective of representing the target domain. Therefore, it is not that effective in identifying trending keywords representative for a target domain compared to our models.

\section{Conclusion}
\label{sec:conclusion}
This paper proposes an approach for solving an important but understudied problem of a domain representative keywords selection from candidates contrasting with a context domain. Our approach utilizes a two-component mixture model mechanism followed by a novel subset selection optimization algorithm to tackle the problem. We believe this work will encourage the automated text structuring problem and help a wide range of downstream applications in NLP. For future research direction, we want to focus on adapting the proposed approach in a more challenging task like single document summarization where the scope of information is limited. Besides, our proposed techniques are general and thus can be used in many applications such as information extraction, topic modeling, and concept indexing. Exploration of those applications is an interesting future direction.
\section*{Acknowledgements}
We thank the anonymous reviewers for their valu- able comments and suggestions. This material is based upon work supported by the National Science Foundation IIS 16-19302 and IIS 16-33755, Zhejiang University ZJU Research 083650, Futurewei Technologies HF2017060011 and 094013, IBM-Illinois Center for Cognitive Computing Systems Research (C3SR)- a research collaboration as part of the IBM Cognitive Horizon Network, grants from eBay and Microsoft Azure, UIUC OVCR CCIL Planning Grant 434S34, UIUC CSBS Small Grant 434C8U, and UIUC New Frontiers Initiative. Any opinions, findings, and conclusions or recommendations expressed in this publication are those of the author(s) and do not necessarily reflect the views of the funding agencies.
\bibliography{anthology}

\begin{thebibliography}{35}
\expandafter\ifx\csname natexlab\endcsname\relax\def\natexlab#1{#1}\fi

\bibitem[{Alzaidy et~al.(2019)Alzaidy, Caragea, and Giles}]{alzaidy2019bi}
Rabah Alzaidy, Cornelia Caragea, and C~Lee Giles. 2019.
\newblock Bi-lstm-crf sequence labeling for keyphrase extraction from scholarly
  documents.
\newblock In \emph{The world wide web conference}, pages 2551--2557.

\bibitem[{Bennani-Smires et~al.(2018)Bennani-Smires, Musat, Hossmann,
  Baeriswyl, and Jaggi}]{bennani2018simple}
Kamil Bennani-Smires, Claudiu Musat, Andreea Hossmann, Michael Baeriswyl, and
  Martin Jaggi. 2018.
\newblock Simple unsupervised keyphrase extraction using sentence embeddings.
\newblock \emph{arXiv preprint arXiv:1801.04470}.

\bibitem[{Berger and Lafferty(1999)}]{DBLP:conf/sigir/BergerL99}
Adam~L. Berger and John~D. Lafferty. 1999.
\newblock \href {https://doi.org/10.1145/312624.312681} {Information retrieval
  as statistical translation}.
\newblock In \emph{{SIGIR} '99: Proceedings of the 22nd Annual International
  {ACM} {SIGIR} Conference on Research and Development in Information
  Retrieval, August 15-19, 1999, Berkeley, CA, {USA}}, pages 222--229. {ACM}.

\bibitem[{Blei et~al.(2003)Blei, Ng, and Jordan}]{blei2003latent}
David~M Blei, Andrew~Y Ng, and Michael~I Jordan. 2003.
\newblock Latent dirichlet allocation.
\newblock \emph{the Journal of machine Learning research}, 3:993--1022.

\bibitem[{Bouma(2009)}]{bouma2009normalized}
Gerlof Bouma. 2009.
\newblock Normalized (pointwise) mutual information in collocation extraction.
\newblock \emph{Proceedings of GSCL}, 30:31--40.

\bibitem[{Brown et~al.(1993)Brown, Della~Pietra, Della~Pietra, and
  Mercer}]{brown1993mathematics}
Peter~F Brown, Stephen~A Della~Pietra, Vincent~J Della~Pietra, and Robert~L
  Mercer. 1993.
\newblock The mathematics of statistical machine translation: Parameter
  estimation.
\newblock \emph{Computational linguistics}, 19(2):263--311.

\bibitem[{Chen et~al.(2017)Chen, Soni, Pappu, and Mehdad}]{chen2017doctag2vec}
Sheng Chen, Akshay Soni, Aasish Pappu, and Yashar Mehdad. 2017.
\newblock Doctag2vec: An embedding based multi-label learning approach for
  document tagging.
\newblock \emph{arXiv preprint arXiv:1707.04596}.

\bibitem[{Dempster et~al.(1977)Dempster, Laird, and
  Rubin}]{dempster1977maximum}
Arthur~P Dempster, Nan~M Laird, and Donald~B Rubin. 1977.
\newblock Maximum likelihood from incomplete data via the em algorithm.
\newblock \emph{Journal of the Royal Statistical Society: Series B
  (Methodological)}, 39(1):1--22.

\bibitem[{Feige(1998)}]{feige1998threshold}
Uriel Feige. 1998.
\newblock A threshold of ln n for approximating set cover.
\newblock \emph{Journal of the ACM (JACM)}, 45(4):634--652.

\bibitem[{Handler et~al.(2016)Handler, Denny, Wallach, and
  O’Connor}]{handler2016bag}
Abram Handler, Matthew Denny, Hanna Wallach, and Brendan O’Connor. 2016.
\newblock Bag of what? simple noun phrase extraction for text analysis.
\newblock In \emph{Proceedings of the First Workshop on NLP and Computational
  Social Science}, pages 114--124.

\bibitem[{H{\"a}tty et~al.(2017)H{\"a}tty, Dorna, and
  im~Walde}]{hatty2017evaluating}
Anna H{\"a}tty, Michael Dorna, and Sabine~Schulte im~Walde. 2017.
\newblock Evaluating the reliability and interaction of recursively used
  feature classes for terminology extraction.
\newblock In \emph{Proceedings of the student research workshop at the 15th
  conference of the European chapter of the association for computational
  linguistics}, pages 113--121.

\bibitem[{Hofmann(2001)}]{hofmann2001unsupervised}
Thomas Hofmann. 2001.
\newblock Unsupervised learning by probabilistic latent semantic analysis.
\newblock \emph{Machine learning}, 42(1):177--196.

\bibitem[{Huang et~al.(2021)Huang, Chang, Xiong, and Hwu}]{huang2021measuring}
Jie Huang, Kevin Chen-Chuan Chang, Jinjun Xiong, and Wen-mei Hwu. 2021.
\newblock Measuring fine-grained domain relevance of terms: A hierarchical
  core-fringe approach.
\newblock In \emph{Proceedings of the 59th Annual Meeting of the Association
  for Computational Linguistics and the 11th International Joint Conference on
  Natural Language Processing}.

\bibitem[{Hughes et~al.(2020)Hughes, Aycock, Caines, Buttery, and
  Hutchings}]{hughes2020detecting}
Jack Hughes, Seth Aycock, Andrew Caines, Paula Buttery, and Alice Hutchings.
  2020.
\newblock Detecting trending terms in cybersecurity forum discussions.
\newblock In \emph{Proceedings of the Sixth Workshop on Noisy User-generated
  Text (W-NUT 2020)}, pages 107--115.

\bibitem[{Karimzadehgan and Zhai(2010)}]{karimzadehgan2010estimation}
Maryam Karimzadehgan and ChengXiang Zhai. 2010.
\newblock Estimation of statistical translation models based on mutual
  information for ad hoc information retrieval.
\newblock In \emph{Proceedings of the 33rd international ACM SIGIR conference
  on Research and development in information retrieval}, pages 323--330.

\bibitem[{Kaushal et~al.(2019)Kaushal, Iyer, Kothawade, Mahadev, Doctor, and
  Ramakrishnan}]{kaushal2019learning}
Vishal Kaushal, Rishabh Iyer, Suraj Kothawade, Rohan Mahadev, Khoshrav Doctor,
  and Ganesh Ramakrishnan. 2019.
\newblock Learning from less data: A unified data subset selection and active
  learning framework for computer vision.
\newblock In \emph{2019 IEEE Winter Conference on Applications of Computer
  Vision (WACV)}, pages 1289--1299. IEEE.

\bibitem[{Kullback and Leibler(1951)}]{kullback1951information}
Solomon Kullback and Richard~A Leibler. 1951.
\newblock On information and sufficiency.
\newblock \emph{The annals of mathematical statistics}, 22(1):79--86.

\bibitem[{Liu et~al.(2015)Liu, Shang, Wang, Ren, and Han}]{liu2015mining}
Jialu Liu, Jingbo Shang, Chi Wang, Xiang Ren, and Jiawei Han. 2015.
\newblock Mining quality phrases from massive text corpora.
\newblock In \emph{Proceedings of the 2015 ACM SIGMOD International Conference
  on Management of Data}, pages 1729--1744.

\bibitem[{Lu et~al.(2019)Lu, Zhou, Yu, and Jia}]{lu2019concept}
Weiming Lu, Yangfan Zhou, Jiale Yu, and Chenhao Jia. 2019.
\newblock \href {https://doi.org/10.1609/aaai.v33i01.33019678} {Concept
  extraction and prerequisite relation learning from educational data}.
\newblock \emph{Proceedings of the AAAI Conference on Artificial Intelligence},
  33(01):9678--9685.

\bibitem[{Meng et~al.(2017)Meng, Zhao, Han, He, Brusilovsky, and
  Chi}]{meng2017deep}
Rui Meng, Sanqiang Zhao, Shuguang Han, Daqing He, Peter Brusilovsky, and
  Yu~Chi. 2017.
\newblock Deep keyphrase generation.
\newblock \emph{arXiv preprint arXiv:1704.06879}.

\bibitem[{Mihalcea and Tarau(2004)}]{mihalcea2004textrank}
Rada Mihalcea and Paul Tarau. 2004.
\newblock Textrank: Bringing order into text.
\newblock In \emph{Proceedings of the 2004 conference on empirical methods in
  natural language processing}, pages 404--411.

\bibitem[{Mikolov et~al.(2013)Mikolov, Chen, Corrado, and
  Dean}]{mikolov2013efficient}
Tomas Mikolov, Kai Chen, Greg Corrado, and Jeffrey Dean. 2013.
\newblock Efficient estimation of word representations in vector space.
\newblock \emph{arXiv preprint arXiv:1301.3781}.

\bibitem[{Minoux(1978)}]{minoux1978accelerated}
Michel Minoux. 1978.
\newblock Accelerated greedy algorithms for maximizing submodular set
  functions.
\newblock In \emph{Optimization techniques}, pages 234--243. Springer.

\bibitem[{Mirchandani and Francis(1990)}]{mirchandani1990discrete}
Pitu~B Mirchandani and Richard~L Francis. 1990.
\newblock \emph{Discrete location theory}.
\newblock New York: Wiley.

\bibitem[{Monroe et~al.(2008)Monroe, Colaresi, and Quinn}]{monroe2008fightin}
Burt~L Monroe, Michael~P Colaresi, and Kevin~M Quinn. 2008.
\newblock Fightin'words: Lexical feature selection and evaluation for
  identifying the content of political conflict.
\newblock \emph{Political Analysis}, 16(4):372--403.

\bibitem[{Pennington et~al.(2014)Pennington, Socher, and
  Manning}]{pennington2014glove}
Jeffrey Pennington, Richard Socher, and Christopher~D Manning. 2014.
\newblock Glove: Global vectors for word representation.
\newblock In \emph{Proceedings of the 2014 conference on empirical methods in
  natural language processing (EMNLP)}, pages 1532--1543.

\bibitem[{Sarawagi et~al.(2003)Sarawagi, Chakrabarti, and
  Godbole}]{sarawagi2003cross}
Sunita Sarawagi, Soumen Chakrabarti, and Shantanu Godbole. 2003.
\newblock Cross-training: Learning probabilistic mappings between topics.
\newblock In \emph{Proceedings of the ninth ACM SIGKDD international conference
  on Knowledge discovery and data mining}, pages 177--186.

\bibitem[{Shang et~al.(2018)Shang, Liu, Jiang, Ren, Voss, and
  Han}]{shang2018automated}
Jingbo Shang, Jialu Liu, Meng Jiang, Xiang Ren, Clare~R Voss, and Jiawei Han.
  2018.
\newblock Automated phrase mining from massive text corpora.
\newblock \emph{IEEE Transactions on Knowledge and Data Engineering},
  30(10):1825--1837.

\bibitem[{Tang et~al.(2008)Tang, Zhang, Yao, Li, Zhang, and
  Su}]{tang2008arnetminer}
Jie Tang, Jing Zhang, Limin Yao, Juanzi Li, Li~Zhang, and Zhong Su. 2008.
\newblock Arnetminer: extraction and mining of academic social networks.
\newblock In \emph{Proceedings of the 14th ACM SIGKDD international conference
  on Knowledge discovery and data mining}, pages 990--998.

\bibitem[{Turney(2000)}]{turney2000learning}
Peter~D Turney. 2000.
\newblock Learning algorithms for keyphrase extraction.
\newblock \emph{Information retrieval}, 2(4):303--336.

\bibitem[{Wang et~al.(2020)Wang, Zhu, Jiang, Zhang, Wang, Ni, Xie, and
  Xiao}]{wang2020mining}
Li~Wang, Wei Zhu, Sihang Jiang, Sheng Zhang, Keqiang Wang, Yuan Ni, Guotong
  Xie, and Yanghua Xiao. 2020.
\newblock Mining infrequent high-quality phrases from domain-specific corpora.
\newblock In \emph{Proceedings of the 29th ACM International Conference on
  Information \& Knowledge Management}, pages 1535--1544.

\bibitem[{Witten et~al.(1999)Witten, Paynter, Frank, Gutwin, and
  Nevill{-}Manning}]{DBLP:conf/dl/WittenPFGN99}
Ian~H. Witten, Gordon~W. Paynter, Eibe Frank, Carl Gutwin, and Craig~G.
  Nevill{-}Manning. 1999.
\newblock \href {https://doi.org/10.1145/313238.313437} {{KEA:} practical
  automatic keyphrase extraction}.
\newblock In \emph{Proceedings of the Fourth {ACM} conference on Digital
  Libraries, August 11-14, 1999, Berkeley, CA, {USA}}, pages 254--255. {ACM}.

\bibitem[{Zhai and Lafferty(2001)}]{zhai2001model}
Chengxiang Zhai and John Lafferty. 2001.
\newblock Model-based feedback in the language modeling approach to information
  retrieval.
\newblock In \emph{Proceedings of the tenth international conference on
  Information and knowledge management}, pages 403--410.

\bibitem[{Zhai et~al.(2004)Zhai, Velivelli, and Yu}]{zhai2004cross}
ChengXiang Zhai, Atulya Velivelli, and Bei Yu. 2004.
\newblock A cross-collection mixture model for comparative text mining.
\newblock In \emph{Proceedings of the tenth ACM SIGKDD international conference
  on Knowledge discovery and data mining}, pages 743--748.

\bibitem[{Zhang et~al.(2018)Zhang, Tao, Chen, Shen, Jiang, Sadler, Vanni, and
  Han}]{zhang2018taxogen}
Chao Zhang, Fangbo Tao, Xiusi Chen, Jiaming Shen, Meng Jiang, Brian Sadler,
  Michelle Vanni, and Jiawei Han. 2018.
\newblock Taxogen: Constructing topical concept taxonomy by adaptive term
  embedding and clustering.
\newblock \emph{Proc. KDDI}.

\end{thebibliography}
\bibliographystyle{acl_natbib}
\clearpage
\appendix
\begin{table}[!tp]
\centering
\resizebox{1.0\linewidth}{!}{%
\begin{tabular}{@{}p{1em}PPPPPPP@{}}
\toprule
$k$   & \multicolumn{1}{c}{RF}      & \multicolumn{1}{c}{LO}      & \multicolumn{1}{c}{PR}      & \multicolumn{1}{c}{FL}      &  \multicolumn{1}{c}{$\mathbf{KL_{rf}}$} & \multicolumn{1}{c}{MM}               & \multicolumn{1}{c}{$\mathbf{KL_{mm}}$} \\ \midrule
\multicolumn{8}{c}{Candidate Keywords from Springer}                                                                                                                                                         \\ \midrule
10  & 2.0691  & 2.1157  & 1.1273  & 1.0852  & 1.0998                                                 & 2.1401           & \textbf{2.1432}                                            \\
20  & 2.1745  & 2.2391  & 2.2557  & 2.1745  & 2.1805                                                 & 2.2915           & \textbf{3.3260}                                            \\
30  & 2.2302  & 3.3938  & 3.4220  & 2.2261  & 2.2556                                                 & 3.4259           & \textbf{3.4262}                                            \\
40  & 2.2873  & 3.5040  & 3.5053  & 2.2847  & 2.3223                                                 & 5.5379           & \textbf{5.5404}                                            \\
50  & 2.3591  & 5.6136  & 3.6136  & 2.3591  & 2.3689                                                 & 5.6346           & \textbf{5.6576}                                            \\
100 & 3.7701  & 7.0986  & 6.0387  & 3.7601  & 3.7846                                                 & \textbf{8.1300}  & 7.1193                                                     \\
200 & 6.4205  & 12.0418 & 8.8759  & 5.3433  & 6.3992                                                 & \textbf{12.1043} & 12.0591                                                    \\
500 & 13.3390 & 19.4237 & 13.9664 & 13.2692 & 12.2365                                                & \textbf{19.4987} & 19.4857                                                    \\ \midrule
\multicolumn{8}{c}{Candidate Keywords from Aminer}                                                                                                                                                           \\ \midrule
10  & 1.1020  & 2.1205  & 1.1101  & 1.1073  & 1.1109                                                 & 2.1205           & \textbf{2.1645}                                            \\
20  & 2.1782  & 2.3050  & 1.2272  & 2.1699  & 2.1956                                                 & 3.3048           & \textbf{3.3079}                                            \\
30  & 2.2433  & 3.4320  & 2.3747  & 2.2508  & 2.2795                                                 & 3.4394           & \textbf{4.4522}                                            \\
40  & 2.3159  & 3.5312  & 3.5345  & 2.3012  & 2.3944                                                 & 5.5693           & \textbf{5.5700}                                            \\
50  & 2.4617  & 5.6519  & 4.6396  & 3.4657  & 2.4675                                                 & 5.6463           & \textbf{5.6829}                                            \\
100 & 3.8432  & 8.1523  & 6.1264  & 3.8038  & 3.8429                                                 & \textbf{9.1788}  & 8.1679                                                     \\
200 & 6.5778  & 12.1533 & 8.9075  & 6.5320  & 6.5431                                                 & \textbf{12.2007} & 12.1013                                                    \\
500 & 13.668  & 19.5232 & 15.1445 & 13.4622 & 12.4942                                                & 19.5580          & \textbf{20.5255}                                           \\ \bottomrule
\end{tabular}}
\caption{Results of selected summary keywords' correspondence with arxiv category keywords}
\label{tab:domain_summ_nmi_sp_am}
\vspace{-5mm}
\end{table}

\section{Proof of Theorem \ref{thm:greedy_opt}}
\label{sec:appendix_proof}
To prove this, we need to first show that $\phi(\cdot)$ is \textit{submodular} and \textit{monotone}. As, we are concerned on minimizing $\phi(\cdot)$, it is equivalent to maximizing $F(\cdot) = -\phi(\cdot)$. Hence, it is sufficient to prove that $F(.)$ is submodular and monotone.
Let, $X\subseteq Y \subseteq V$ and $v \in V$, then we get
\resizebox{1.0\linewidth}{!}{
\begin{minipage}{\linewidth}
\begin{align*}
    &F(X \cup \{v\}) - F(\{X\})  = \\ &\sum_{t_i\in V} p(t_i) \log \frac{\sum_{t_j \in X\cup \{v\}}p(t_i|t_j)p(t_j)}{p(t_i)}
    \\ &- \sum_{t_i\in V} p(t_i) \log \frac{\sum_{t_j \in X}p(t_i|t_j)p(t_j)}{p(t_i)}\\ 
    &= \sum_{t_i\in V} p(t_i) \log \frac{\sum_{t_j \in X\cup \{v\}}p(t_i|t_j)p(t_j)}{\sum_{t_j \in X}p(t_i|t_j)p(t_j)}\\
    &= \sum_{t_i\in V} p(t_i) \log \frac{\sum_{t_j \in X}p(t_i|t_j)p(t_j) + p(t_i|v)p(v)}{\sum_{t_j \in X}p(t_i|t_j)p(t_j)}\\
    &= \sum_{t_i\in V} p(t_i) \log (1+\frac{p(t_i|v)p(v)}{\sum_{t_j \in X}p(t_i|t_j)p(t_j)}).\\
    &\text{Similarly, }
    F(Y \cup \{v\}) - F(\{Y\})  = \\&\sum_{t_i\in V} p(t_i) \log (1+\frac{p(t_i|v)p(v)}{\sum_{t_j \in Y}p(t_i|t_j)p(t_j)}).\\
    & \text{As } X\subseteq Y, \text{then, }\\ 
    &\sum_{t_j \in X}p(t_i|t_j)p(t_j) \leq \sum_{t_j \in Y}p(t_i|t_j)p(t_j).\\
\end{align*}
\end{minipage}
}
Therefore,  $F(X \cup \{v\}) - F(\{X\}) \geq F(Y \cup \{v\}) - F(\{Y\})$ which proves that $F(\cdot)$ is \textit{submodular}.
Moreover, we can show that $F(Y) - F(X) = \sum_{t_i \in V} p(t_i)\log\frac{\sum_{t_j \in Y}p(t_i|t_j)p(t_j)}{\sum_{t_j \in X}p(t_i|t_j)p(t_j)} \geq 1$. Hence, $F(Y) \geq F(X) \text{ for } X\subseteq Y\subseteq V$ which proves that $F(\cdot)$ is monotone. Therefore, it proves that minimizing $\phi(\cdot)$ using simple greedy algorithm guarantees an approximation of $1 - \frac{1}{e}$.


\begin{table}[!tp]
\tiny
\resizebox{1.0\linewidth}{!}{%
\begin{tabular}{p{0.06\linewidth}p{0.74\linewidth}}
\toprule
 & \multicolumn{1}{c}{Selected 20 keywords in CS} \\ \midrule
RF                                            & model, method, paper, problem, approach, \textbf{algorithm}, datum, \textbf{network}, \textbf{system}, \textbf{performance}, task, result, \textbf{image}, number, user, application, time, \textbf{dataset}, \textbf{graph}, work                                                                               \\ \midrule
LO                                            & task, \textbf{algorithm}, user, \textbf{network}, \textbf{performance}, \textbf{dataset}, \textbf{image}, problem, \textbf{training}, approach, \textbf{deep learning}, method, \textbf{node}, \textbf{agent}, language, \textbf{neural network}, paper, video, challenge, \textbf{architecture}                                            \\ \midrule
PR                                            & \textbf{image}, \textbf{graph}, \textbf{dataset}, user, method, model, \textbf{network}, task, \textbf{algorithm}, datum, problem, \textbf{system}, training, channel, node, performance, object, agent, deep learning, language                                                       \\ \midrule
FL                                            & \textbf{image}, model, paper, problem, method, network, datum, approach, \textbf{algorithm}, \textbf{system}, user, performance, result, \textbf{graph}, task, application, number, work, time, dataset                                                                                  \\ \midrule
{$\mathbf{KL_{rf}}$} & model, method, problem, \textbf{system}, \textbf{network}, datum, \textbf{algorithm}, paper, \textbf{image}, task, user, approach, performance, graph, application, dataset, time, result, number, information                                                                                    \\ \midrule
MM                                            &task, \textbf{algorithm}, user, \textbf{network}, \textbf{performance}, \textbf{dataset}, image, \textbf{training}, \textbf{deep learning}, \textbf{node}, \textbf{agent}, language, \textbf{neural network}, video, \textbf{architecture}, challenge, \textbf{robot}, real world, \textbf{attack}, \textbf{learning} \\ \midrule
{$\mathbf{KL_{mm}}$} & task, user, \textbf{algorithm}, \textbf{dataset}, \textbf{network}, \textbf{performance}, \textbf{image}, \textbf{training}, \textbf{agent}, language, \textbf{deep learning}, \textbf{attack}, \textbf{robot}, video, challenge, \textbf{node}, \textbf{neural network}, \textbf{query}, \textbf{code}, \textbf{machine learning}  \\
\bottomrule
\end{tabular}}
{\raggedright \centering Keywords distinctive to the CS domain are \textbf{highlighted} (annotated by authors). \par}
\vspace{-1mm}
\caption{Summary keywords selected from Springer candidate keywords in CS domain}
\label{tab:cs_summary_sp}
\vspace{-1mm}
\end{table}

\begin{table}[!tp]
\tiny
\resizebox{1.0\linewidth}{!}{%
\begin{tabular}{p{0.06\linewidth}p{0.74\linewidth}}
\toprule
 & \multicolumn{1}{c}{Selected 20 keywords in CS} \\ \midrule
RF                                            & model, method, \textbf{network}, \textbf{algorithm},   \textbf{system}, datum, problem, user, \textbf{image}, time, \textbf{graph}, application, \textbf{performance},   state, feature, \textbf{dataset}, number, art, work, information                           \\ \midrule
LO                                            & \textbf{network}, user, \textbf{algorithm}, \textbf{dataset}, art, \textbf{performance}, training,   \textbf{image}, task, \textbf{deep learning}, \textbf{node}, \textbf{learning}, \textbf{agent}, \textbf{attack}, \textbf{neural network},   language, video, problem, \textbf{robot}, \textbf{graph}                 \\ \midrule
PR                                            & \textbf{image}, art, \textbf{graph}, \textbf{dataset}, state, user, model, \textbf{network},   method, \textbf{algorithm}, datum, vertex, training, channel, \textbf{system}, feature, \textbf{node},   \textbf{deep learning}, experiment, object                           \\ \midrule
FL                                            & art, model, method, \textbf{network}, \textbf{algorithm}, \textbf{image}, datum, \textbf{system},   problem, user, \textbf{graph}, application, \textbf{performance}, time, work, number, \textbf{dataset},   feature, order, experiment                            \\ \midrule
$\mathbf{KL_{rf}}$ & model, method, \textbf{algorithm}, \textbf{network}, \textbf{system}, datum, \textbf{image}, user,   problem, \textbf{graph}, \textbf{dataset}, application, \textbf{performance}, time, \textbf{feature}, \textbf{agent}, art,   information, number, \textbf{training}                       \\ \midrule
MM                                            & \textbf{network}, user, \textbf{algorithm}, \textbf{dataset}, art, \textbf{performance}, \textbf{training},   \textbf{image}, task, \textbf{deep learning}, \textbf{node}, \textbf{learning}, \textbf{agent}, \textbf{neural network}, \textbf{attack},   language, video, \textbf{robot}, \textbf{architecture}, \textbf{machine learning} \\ \midrule
$\mathbf{KL_{mm}}$ & \textbf{dataset}, user, \textbf{network}, \textbf{algorithm}, \textbf{training}, \textbf{image}, \textbf{agent},   task, \textbf{performance}, \textbf{attack}, art, language, \textbf{deep learning}, \textbf{robot}, \textbf{learning},   video, \textbf{node}, \textbf{machine learning}, \textbf{code}, \textbf{neural network}  \\
\bottomrule
\end{tabular}}
{\raggedright \centering Keywords distinctive to the CS domain are \textbf{highlighted} (annotated by authors). \par}
\vspace{-1mm}
\caption{Summary keywords selected from Aminer candidate keywords in CS domain}
\label{tab:cs_summary_am}
\vspace{-1mm}
\end{table}

\begin{table}[!tp]
\resizebox{1.0\linewidth}{!}{%
\begin{tabular}{@{}l|l|l|l@{}}
\toprule
Religion                                                                                          & Recreation                                                                                                  & Science                                                                                   & Politics                                                                                                \\ \midrule
\begin{tabular}[c]{@{}l@{}}talk.religion.misc\\ alt.atheism\\ soc.religion.christian\end{tabular} & \begin{tabular}[c]{@{}l@{}}rec.autos\\ rec.motorcycles\\ rec.sport.baseball\\ rec.sport.hockey\end{tabular} & \begin{tabular}[c]{@{}l@{}}sci.crypt\\ sci.electronics\\ sci.med\\ sci.space\end{tabular} & \begin{tabular}[c]{@{}l@{}}talk.politics.misc\\ talk.politics.guns\\ talk.politics.mideast\end{tabular} \\ \bottomrule
\end{tabular}}
\vspace{-1mm}
\caption{Subtopics in each of the four topics in the newsgroup dataset}
\label{tab:subtopics_ng}
\vspace{-3mm}
\end{table}

\begin{table*}[!tp]
\scriptsize
\centering
\resizebox{1.0\linewidth}{!}{%
\begin{tabular}{p{0.05\textwidth}p{0.2\textwidth}p{0.2\textwidth}p{0.2\textwidth}p{0.2\textwidth}}
\toprule
                                                           & \multicolumn{1}{c}{Religion}                                                                                                                                                                             & \multicolumn{1}{c}{Recreation}                                                                                                                                        & \multicolumn{1}{c}{Science}                                                                                                                                                                    & \multicolumn{1}{c}{Politics}                                                                                                                                                                  \\ \midrule
RF                                                         & thing, \textbf{church}, life, word, man, \textbf{religion}, \textbf{bible}, \textbf{faith},   question, \textbf{belief}, book, point, law, evidence, \textbf{sin}, reason, world, truth,   child, \textbf{god}                                                          & \textbf{game}, \textbf{car}, \textbf{team}, \textbf{player}, \textbf{bike}, \textbf{season}, \textbf{point}, \textbf{hockey}, problem,   lot, \textbf{goal}, \textbf{baseball}, guy, \textbf{engine}, power, number, year, line, question, \textbf{run}                           & key, information, thing, government, \textbf{space}, \textbf{encryption}, datum,   \textbf{clipper}, \textbf{chip}, case, number, phone, \textbf{bit}, \textbf{privacy}, drug, earth, power,   \textbf{security}, \textbf{program}, disease                            & \textbf{government}, \textbf{gun}, child, state, law, country, man, \textbf{president},   case, \textbf{war}, group, fact, \textbf{firearm}, number, \textbf{crime}, question, \textbf{weapon}, world,   history, population                                 \\\midrule
LO                                                         & \textbf{church}, \textbf{bible}, \textbf{faith}, \textbf{religion}, \textbf{belief}, \textbf{sin}, \textbf{god}, \textbf{scripture},   life, word, \textbf{atheist}, truth, \textbf{atheism}, \textbf{homosexuality}, love, man, evidence, son,   \textbf{morality}, book                                            & \textbf{game}, \textbf{team}, \textbf{car}, \textbf{player}, \textbf{bike}, \textbf{season}, \textbf{hockey}, \textbf{baseball},   \textbf{playoff}, \textbf{engine}, \textbf{goal}, \textbf{pitcher}, \textbf{tire}, \textbf{run}, pen, \textbf{league}, \textbf{puck}, \textbf{motorcycle},   dog, \textbf{clutch}                    & key, \textbf{encryption}, \textbf{space}, \textbf{clipper}, \textbf{privacy}, \textbf{satellite}, \textbf{mission},   \textbf{disease}, \textbf{shuttle}, phone, \textbf{orbit}, \textbf{escrow}, \textbf{moon}, \textbf{cancer}, \textbf{algorithm}, \textbf{spacecraft},   \textbf{security}, launch, \textbf{vitamin}, health               & \textbf{gun}, \textbf{government}, \textbf{firearm}, \textbf{president}, country, \textbf{weapon}, \textbf{crime},   village, \textbf{soldier}, \textbf{genocide}, \textbf{war}, population, state, child, police, \textbf{turk},   \textbf{massacre}, \textbf{handgun}, compound, new york               \\ \midrule
PR                                                         & man, word, thing, life, world, history, \textbf{church}, book,   question, \textbf{bible}, \textbf{faith}, point, truth, reason, matter, law, year, \textbf{religion},   earth, mind                                                         & \textbf{game}, \textbf{team}, \textbf{player}, \textbf{car}, \textbf{season}, \textbf{goal}, \textbf{point}, \textbf{hockey}, shot,   year, power, number, \textbf{engine}, \textbf{bike}, \textbf{win}, \textbf{league}, speed, line, \textbf{run}, end                                   & information, datum, year, study, number, united states, \textbf{space},   nature, \textbf{security}, mail, government, thing, \textbf{encryption}, case, archive, key,   life, book, law, \textbf{science}                         & \textbf{government}, man, group, \textbf{war}, world, village, child, fact,   year, history, life, state, house, end, woman, home, \textbf{power}, arm, law,   population                                                \\ \midrule
FL                                                         & man, thing, \textbf{church}, life, word, \textbf{religion}, book, question,   history, point, \textbf{bible}, law, evidence, reason, \textbf{sin}, world, \textbf{belief}, child,   \textbf{faith}, case                                                       & \textbf{game}, \textbf{car}, \textbf{team}, \textbf{player}, \textbf{bike}, \textbf{season}, pit, problem, \textbf{hockey},   lot, power, \textbf{point}, \textbf{baseball}, \textbf{engine}, \textbf{goal}, \textbf{run}, question, guy, standing,   \textbf{speed}                       & information, \textbf{space}, key, study, thing, government, datum,   year, case, number, patient, \textbf{software}, power, book, archive, food, \textbf{clipper},   \textbf{mission}, hicnet medical newsletter page, \textbf{encryption}  & \textbf{government}, man, village, \textbf{gun}, \textbf{president}, \textbf{sumgait}, history,   state, case, child, \textbf{law}, world, country, population, fact, \textbf{war}, los angeles,   number, group, year                              \\ \midrule
{$\mathbf{KL_{rf}}$}     & thing, life, \textbf{church}, word, \textbf{belief}, man, question, \textbf{religion},   \textbf{bible}, \textbf{faith}, book, law, evidence, point, \textbf{sin}, world, reason, child, truth,   \textbf{god}                                                          & \textbf{game}, \textbf{car}, \textbf{team}, \textbf{bike}, \textbf{player}, \textbf{point}, problem, \textbf{season}, lot,   \textbf{hockey}, \textbf{engine}, guy, power, \textbf{baseball}, question, \textbf{goal}, number, list, road,   year                        & key, information, \textbf{space}, thing, government, \textbf{encryption}, datum,   case, \textbf{clipper}, number, \textbf{chip}, \textbf{disease}, power, phone, \textbf{earth}, drug, \textbf{program},   \textbf{bit}, book, \textbf{privacy}                                & \textbf{government}, \textbf{gun}, child, state, country, man, law, president,   \textbf{war}, case, group, fact, question, \textbf{crime}, population, number, world, \textbf{firearm},   history, woman                                  \\ \midrule
MM                                                         & \textbf{bible}, \textbf{church}, \textbf{faith}, \textbf{sin}, \textbf{scripture}, \textbf{atheism}, \textbf{god}, \textbf{gospel},   \textbf{christianity}, \textbf{prophecy}, \textbf{mcconkie}, \textbf{jesus christ}, \textbf{prophet}, \textbf{new testament},   \textbf{atheist}, \textbf{disciple}, \textbf{holy spirit}, \textbf{theist}, \textbf{christian}, \textbf{homosexuality} & \textbf{team}, pen, \textbf{player}, \textbf{bike}, \textbf{hockey}, \textbf{season}, \textbf{second period}, \textbf{puck},   \textbf{playoff}, \textbf{first period}, \textbf{ranger}, \textbf{schedule}, \textbf{pitcher}, \textbf{baseball}, \textbf{nhl}, \textbf{cub}, \textbf{tire},   \textbf{injury}, \textbf{league}, respect & \textbf{encryption}, \textbf{clipper}, \textbf{privacy}, \textbf{satellite}, \textbf{shuttle}, \textbf{orbit},   \textbf{vitamin}, \textbf{infection}, \textbf{escrow}, \textbf{moon}, \textbf{pgp}, \textbf{mission}, \textbf{spacecraft}, \textbf{cryptography},   \textbf{cancer}, \textbf{circuit}, \textbf{astronaut}, \textbf{asteroid}, \textbf{cipher}, \textbf{telescope} & \textbf{gun}, \textbf{firearm}, \textbf{soldier}, village, \textbf{genocide}, \textbf{bayonet}, \textbf{turk},   \textbf{handgun}, \textbf{massacre}, new york, \textbf{tartar}, \textbf{homicide}, \textbf{civilian}, \textbf{weapon}, \textbf{human right},   \textbf{gun control}, \textbf{bullet}, \textbf{troop}, \textbf{ottoman}, \textbf{sumgait}       \\ \midrule
{$\mathbf{KL_{mm}}$} & \textbf{bible}, \textbf{church}, \textbf{faith}, \textbf{sin}, \textbf{scripture}, \textbf{atheism}, \textbf{god}, \textbf{gospel},   \textbf{jesus christ}, \textbf{prophecy}, \textbf{christianity}, \textbf{new testament}, \textbf{mcconkie}, \textbf{prophet}, \textbf{holy   spirit}, \textbf{morality}, \textbf{theist}, \textbf{disciple}, \textbf{homosexuality}, \textbf{atheist}  & \textbf{team}, pen, \textbf{bike}, \textbf{player}, \textbf{season}, \textbf{hockey}, \textbf{second period},   \textbf{playoff}, \textbf{pitcher}, \textbf{puck}, \textbf{schedule}, \textbf{ranger}, \textbf{cub}, \textbf{baseball}, \textbf{tire}, \textbf{clutch}, \textbf{first   period}, favor, respect, \textbf{nhl}  & \textbf{encryption}, \textbf{satellite}, \textbf{clipper}, \textbf{vitamin}, \textbf{shuttle}, \textbf{infection},   \textbf{orbit}, \textbf{moon}, \textbf{cancer}, \textbf{privacy}, \textbf{spacecraft}, \textbf{circuit}, \textbf{mission}, \textbf{pgp}, \textbf{escrow},   \textbf{allergy}, \textbf{yeast}, \textbf{cryptography}, \textbf{diet}, \textbf{solar sail}      & \textbf{gun}, village, \textbf{firearm}, \textbf{soldier}, \textbf{genocide}, \textbf{turk}, new york,   \textbf{massacre}, \textbf{human right}, \textbf{handgun}, \textbf{bayonet}, \textbf{civilian}, \textbf{croat}, \textbf{tartar}, \textbf{weapon}, \textbf{gun   control}, \textbf{troop}, \textbf{homicide}, \textbf{well regulated}, \textbf{sumgait} \\ \bottomrule
\end{tabular}}
{\raggedright \centering Keywords distinctive to the subtopics in a respected topic are \textbf{highlighted} (annotated by authors). \par}
\caption{Summary keywords selected by different algorithms on four topics from newsgroups dataset}
\label{tab:newsgroup_summary}
\end{table*}

\section{Additional Quantitative Results on CS Domain}
\label{sec:additional_quancs}
Results are shown in Table \ref{tab:domain_summ_nmi_sp_am}.

\section{Additional Qualitative Results on CS Domain}
\label{sec:additional_cs}
Results are shown in Tables \ref{tab:cs_summary_sp} and \ref{tab:cs_summary_am}.

\section{Evaluation Using Newsgroup Dataset}
\label{sec:newsgroup_result}
We use newsgroup dataset covering four known topics named \textit{Religion}, \textit{Recreation}, \textit{Science} and \textit{Politics}. In this study, we split the whole dataset into these four topic groups represented by their corpus and use the whole newsgroup dataset as our background corpus. Table \ref{tab:subtopics_ng} shows the subtopics for each of the four topics. For each topic, we show the selected top 20 keywords using different algorithms in Table \ref{tab:newsgroup_summary}.
This study aims to evaluate the capability of the proposed models to select distinctive keywords for each topic compared to the baselines.
We can see that almost all the keywords selected by our methods MM and $\mathbf{KL_{mm}}$ are distinctive for each topic relating closely with respected subtopics shown in Table \ref{tab:subtopics_ng} and do not overlap with other topics. Similarly, as previously, the results from LO come close to ours with some anomalies. For instance, our methods select informative keywords like \textit{jesus christ}, \textit{holy spirit} and \textit{new testament} for \textit{religion} topic rather than non-distinctive keywords like \textit{word}, \textit{man} and \textit{son}.

\end{document}